\documentclass[11pt]{article}

\usepackage[letterpaper,margin=1in]{geometry}

\usepackage[utf8]{inputenc}
\usepackage[T1]{fontenc}
\usepackage{natbib}
\usepackage{hyperref}
\usepackage{url}
\usepackage{booktabs}
\usepackage{amsmath,amsfonts,amssymb,amsthm}
\usepackage{mathtools}
\usepackage{nicefrac}
\usepackage{microtype}
\usepackage{xcolor}
\usepackage{bm}
\usepackage[nameinlink,capitalise,noabbrev]{cleveref}
\usepackage{tikz}
\usepackage{adjustbox}
\usepackage{xifthen}

\hypersetup{colorlinks=true,citecolor=blue,linkcolor=blue,urlcolor=blue}
\renewcommand{\And}{\and}

\newtheorem{theorem}{Theorem}

\newtheorem{corollary}[theorem]{Corollary}

\theoremstyle{definition}

\newtheorem{remark}[theorem]{Remark}

\newcommand{\N}{\mathbb{N}}
\newcommand{\R}{\mathbb{R}}
\newcommand{\zo}{\{0,1\}}
\newcommand{\msf}{\mathsf}

\newcommand{\PM}{\mathsf{PM}}
\newcommand{\HIST}{\mathsf{HIST}}

\newcommand{\emb}{\mathsf{emb}}

\newcommand{\tail}{\mathsf{tail}}
\newcommand{\hist}{\mathsf{hist}}

\newcommand{\concat}{\mathbin{\|}}
\renewcommand{\vec}[1]{\boldsymbol{#1}}

\title{ Rethinking the Role of Positional Encoding: \\  {\Large Sliding-Window Transformers without PE  Remain Turing Complete}}

\newif\ifshowscomments

\showscommentstrue

\ifshowscomments
  \newcommand{\warn}[1]{[\textcolor{red}{#1}]}
  \newcommand{\Xnote}[1]{[Xinyu: \textcolor{orange}{#1}]}
\else
  \newcommand{\warn}[1]{}
  \newcommand{\Xnote}[1]{}
\fi

\author{%
  Qian Li \\
  Shenzhen Research Institute of Big Data \\
  \texttt{liqian.ict@gmail.com} \\
  \And
  Xinyu Mao \\
  University of Southern California\\
  \texttt{xinyumao.tcs@gmail.com} \\
  \And
  Shang-Hua Teng\\
  University of Southern California\\
  \texttt{shanghua@usc.edu} \\
}

\begin{document}

\maketitle
\begin{abstract}
Positional encoding (PE) is widely viewed as necessary for transformers to process ordered sequences: without them, the next-token map appears permutation-invariant in its context tokens. This intuition underlies all prior universality results, which rely on positional information to prove that transformers with chain-of-thought can perform arbitrary computation, i.e., they are Turing complete. We revisit this belief in the regime most relevant to long-form reasoning, where generation proceeds through a finite sliding context window. 
Our opening perception is that the window mechanism itself (mildly) breaks the permutation symmetry. 
To distill and precisely capture the degree of this added expressiveness, we introduce an abstract autoregressive model, the HIST model, in which each update depends only on constant-size internal state and the token-count histogram within the current window. 
We prove that this HIST model is Turing complete by showing that the evolution of the window can reveal the token that has just left the window, which suffices to simulate Turing-complete Post machines. We then construct a sliding-window transformer over a constant-size token alphabet, without PE, and show that it can simulate the HIST model. Our result demonstrates that positional encodings are not indispensable for transformers to perform universal computation: The window sliding itself already breaks permutation symmetry and captures sufficient positional information. 
\end{abstract}

\section{Introduction}
Transformer-based large language models with chain-of-thought (CoT) reasoning have demonstrated exceptional performance across a wide range of complex reasoning tasks, such as mathematical problem solving \citep{imo} and code generation \citep{codex,claude47}. These developments have motivated a growing body of theoretical work studying the reasoning abilities of Transformers through the lens of expressiveness, e.g. \citep{perez,ashish2,ashish,tengyu,nowak2024representational,prompt,zubic,depth,pencil,constant,tm2}. A central conclusion of this line of research is that, once CoT reasoning is allowed, transformers can achieve universal computation, that is, they are Turing complete. 

Yet existing universality results rely crucially on positional encoding (PE), which augments each token with an explicit representation of its position in the sequence, thereby enabling the model to distinguish different orderings of the same context.
This view is natural when the model attends to the entire past.
Transformer models consist of interleaved self-attention and feed-forward layers, which are both \emph{order-invariant} and thus cannot distinguish permutations of the same context. 
Therefore, without positional information, the model lacks access to order-dependent information required for structured memory and sequential computation as in Turing machines. From this perspective, positional encodings in transformer models appear necessary for universal computation.


However, this perspective may overlook a key feature of autoregressive reasoning with CoT. In earlier settings without CoT, the context window can effectively cover the entire past. In contrast, CoT, particularly in long-form reasoning, introduces extended intermediate computations that quickly exceed the context capacity, forcing the model to operate with a relatively short 
sliding window rather than attend to the full history. 
As the generative process proceeds, this window slides forward: new tokens enter, old tokens leave, 
and the model attends only to the current window. 
With sliding-window processing, the context is no longer a stack of tokens. 
The motion of the window itself introduces a temporal asymmetry, however small, providing a new source of sequential information. 
Even if the next token generation is permutation-invariant within each individual window, the temporal evolution of the window is order-sensitive: it depends on which tokens enter and exit at each step. This leads to the central question of the paper:
\begin{center}
\emph{Are PEs truly necessary for universal computation in sliding-window transformers?}
\end{center}
Besides theoretical curiosity, we are interested in the problem because PE is not merely a neutral implementation detail, but can influence practical efficiency. 
Depending on the design, they may add extra parameters (e.g., learned absolute PEs \citep{vaswani2017attention}), incur additional overhead in the attention computation (e.g., relative PEs \citep{shaw2018self,su2024roformer}), and complicate length extrapolation and context extension \citep{kazemnejad2023impact}. 
Besides practical considerations, strong PEs can also act as an additional source of computational power, making it less clear whether the source of universal computation lies in the transformer’s content-based computation or in the PE attached to it.
 
\subsection{Our contribution}
We prove that the sliding-window transformer model without PE remains Turing-complete. Our results demonstrate that positional encoding is not necessary for transformers to retrieve sequential structure and achieve universal computation. 
Conceptually, our result separates two notions that are often conflated: although a transformer with PE is permutation-invariant with respect to the tokens inside a single window, the FIFO-like motion of a sliding window is dynamic and breaks permutation-invariance. 
Temporal motion can itself provide sufficient sequential structure for universal computation.

\paragraph{The HIST and parity-HIST model.}

To frame this viewpoint more precisely, we introduce an abstract autoregressive \emph{HIST model} that can only use permutation-invariant summary of data in its context window. 
A HIST machine repeatedly generates the next token using only a finite control state and the histogram of the current window, namely the vector of token counts indexed by the alphabet. It has no access to the positions of individual tokens: conditioned on its control state, the machine's
next-token rule is a permutation-invariant function of the tokens in the current window.

We also consider a more restricted variant, called the \emph{parity-HIST model}, in which the transition function depends on the current window only through the parity of the histogram; namely, it observes only the parity of the number of occurrences of each token type in the current window.

\begin{theorem}[High Level]\label{thm:hist}
The parity-HIST model is Turing complete.
\end{theorem}

Our proof proceeds by simulating Post machines \citep{post1} with parity-HIST
machines.  
A Post machine is an automaton equipped with a queue, and it is
Turing complete.  We view the sliding window of a parity-HIST machine as a
queue-like data structure and show that it can implement the queue operations
needed for simulating Turing machines.

The key point is that the ``dequeue'' operation can be simulated without ever
addressing the position of the front token explicitly.  
Although a parity-HIST machine only observes a permutation-invariant summary of the window, namely the parity of the number of occurrences of each token type, this information is sufficient to recover the token that has just left the window.  
Indeed, by comparing the parity histograms before and after a window update, while temporarily storing in its local state the newest tokens appended at the rear, the machine can infer the dropped front token.  This allows the sliding window to emulate the queue of a Post machine while using only parity information about the current window.



\paragraph{From parity-HIST to Transformers without PE.}
We then show that a sliding-window transformer without positional encodings can simulate the parity-HIST model. The construction uses a constant-size
token alphabet and a constant-size set of learnable parameters. The only non-constant numerical resource used in our implementation is the arithmetic precision needed for computing parity, which we isolate in Remark \ref{rem:precision}.

\begin{theorem}[High Level]
\label{thm:tf_simulates_hist}
 Every parity-HIST machine can be simulated by a sliding-window transformer without using PEs. Consequently, sliding-window transformers without positional encodings are Turing complete.
\end{theorem}

\begin{figure}[!htp]
\centering
\noindent\maxsizebox{.9\linewidth}{!}{
\begin{varwidth}{10\linewidth}
\begin{tikzpicture}[
    font=\small\sffamily,
    token/.style={draw, rectangle, minimum width=1.6cm, minimum height=1.1cm, align=center, fill=white},
    drop_token/.style={token, fill=red!15, dashed, thick},
    add_token/.style={token, fill=green!15, thick},
    arrow/.style={->, >=stealth, thick},
    window/.style={draw=black!60, dashed, thick, rounded corners, inner sep=12pt},
    mathnode/.style={align=center, text width=2.5cm}
]

\node (label_prev) at (-2.5, 2.5) {\textbf{Step $i-1$}};

\node[drop_token] (t1_prev) at (0, 2.5) {$v_{i-s}$\\\textit{(head)}};
\node[token]      (t2_prev) at (1.8, 2.5) {$v_{i-s + 1}$};
\node[token, draw=none] (dots_prev) at (3.6, 2.5) {$\dots$};
\node[token]      (t3_prev) at (5.4, 2.5) {$v_{i-2}$};
\node[token]      (t4_prev) at (7.2, 2.5) {$v_{i-1}$};

s
\node[mathnode] (mu_prev) at (10.5, 2.5) {$\mu_{i-1} = \bigoplus_{k = i - s}^{i - 1} v_k$};

\node (label_curr) at (-2.5, 0) {\textbf{Step $i$}};

\node[token]      (t1_curr) at (1.8, 0) {$v_{i-s + 1}$};
\node[token, draw=none] (dots_curr) at (3.6, 0) {$\dots$};
\node[token]      (t2_curr) at (5.4, 0) {$v_{i-2}$};
\node[token]      (t3_curr) at (7.2, 0) {$v_{i-1}$};
\node[add_token]  (t4_curr) at (9.0, 0) {$v_{i}$\\\textit{(new)}};

\node[mathnode] (mu_curr) at (12.2, 0) {$\mu_i =\bigoplus_{k = i - s + 1}^i v_k$};

\draw[->, dotted, thick, gray] (t2_prev.south) -- (t1_curr.north);
\draw[->, dotted, thick, gray] (dots_prev.south) -- (dots_curr.north);
\draw[->, dotted, thick, gray] (t3_prev.south) -- (t2_curr.north);
\draw[->, dotted, thick, gray] (t4_prev.south) -- (t3_curr.north);

\draw[->, dashed, thick, red] (t1_prev.south) -- ++(0,-0.8) node[below] {\textit{Falls out of window}};

\node[draw, thick, rectangle, fill=blue!10, rounded corners, minimum width=6cm, minimum height=1.5cm, align=center] 
    (mlp) at (10.5, -3) {\textbf{MLP}\\ $\sigma_{i - s} = \sigma_{i} \oplus \mu_{i-1} \oplus  \mu_i$};

\draw[arrow, rounded corners] (mu_prev.south) -- ++(0,-0.5) -| (mlp.north);
\draw[arrow, rounded corners] (mu_curr.south) -- ++(0,-0.5) -| (mlp.north);

\draw[arrow, color=green!60!black, dashed] (t4_curr.south) -- ++(0, -0.5) -| (mlp.north);

\node[font=\bfseries] (output) at (10.5, -5) {queue head $\sigma_{i - s}$};
\draw[arrow, ultra thick, blue!80!black] (mlp.south) -- (output.north);

\end{tikzpicture}
\end{varwidth}}
\caption{Illustration of the sliding window mechanism natively simulating a Post machine queue without positional encodings. By computing the parity of the context window, the Parity-Hist machine can recover the dropped token (the queue head). 
For illustration and simplicity, we think of a token as a bit-string in the figure where $\oplus$ denotes bit-wise XOR. 
}
\label{fig:sliding_window}
\end{figure}
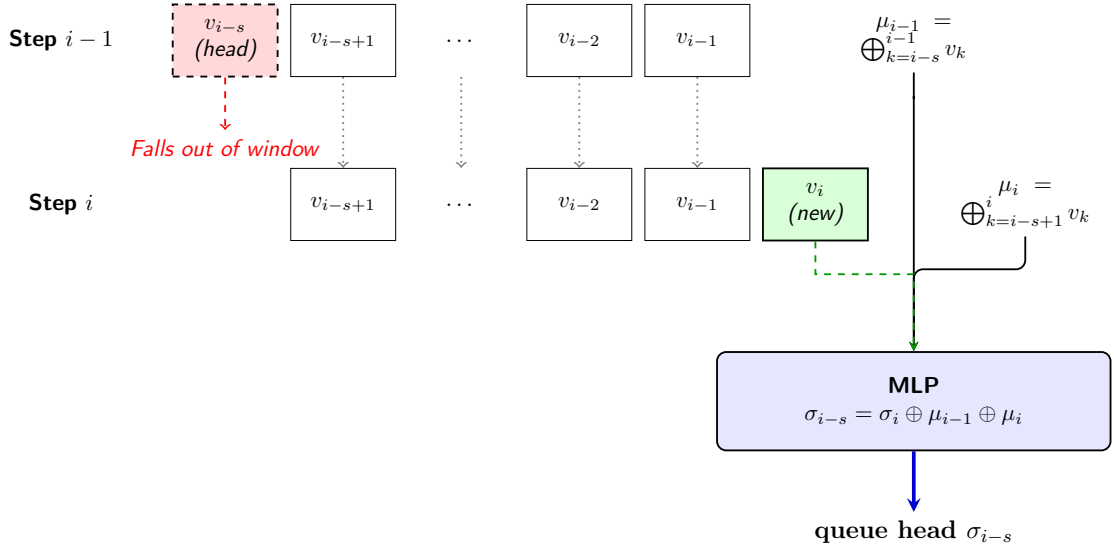

This improves upon previous universality constructions for transformers by
removing the need for positional encodings. Conceptually, the result
separates order information from positional encodings. Universal computation
does require some source of sequential structure, but this structure need not be provided by an externally supplied positional address system.  In the
sliding-window regime, it can instead arise from the temporal dynamics of the
window itself. A few remarks are in order.

\begin{remark}[What the construction does and does not show]
Our construction should not be interpreted as showing that transformers without PE can recover exact token positions inside the context window.  The simulation deliberately avoids addressing individual positions.  

Specifically, let $W_t$ denote the current window, let $x_t$ be the token that has just entered the window, and let $d_t$ be the token that has just left the window. In the parity-HIST model and our transformer construction, the next-token rule can be written as a function of $\msf{hist}(W_t)\bmod 2$,  $\msf{hist}(W_{t-1})\bmod 2$, and $x_t$. Equivalently, \emph{it is a function of $d_t$, $x_t$, and $\msf{hist}(W_t\setminus\{x_t\}) \bmod 2$}. Hence, the next-token rule is invariant under arbitrary permutations of the internal tokens of \(W_t\setminus\{x_t\}\), and is blind to the absolute/relative position of $d_t$.

So, our construction shows that exact positional access is not necessary for universal computation: the combination of a permutation-invariant window summary, a finite control state, and the FIFO-like motion of the sliding window already provides enough sequential structure to simulate arbitrary computation.
\end{remark}

\begin{remark}[Adaptive window updates]
Our construction uses a simple adaptive window-update mechanism. 
The window size is initialized to be the input length. Each generated token is augmented with a control bit that determines how the context window is updated. One value of the bit implements the usual sliding update: the new token enters the window and the oldest token leaves. The other value expands the window size by one: the new token enters the window while the oldest token is retained. Thus, the model can decide during generation whether to preserve or increase the window size.

This adaptive mechanism avoids fixing a worst-case window size in advance.  A fixed-window construction typically requires a window length chosen as a function of the input length, such as an assumed space bound $s(n)$ \citep{constant,tm2}.  Such a bound may be unknown a priori, and even when it is known, it reflects a worst-case requirement over all inputs of length $n$.  By contrast, an adaptive window can grow according to the needs of the particular computation being simulated, making the required context length instance-dependent rather than worst-case over the input length.
\end{remark}

\begin{remark}[Precision]\label{rem:precision}
Our construction uses a constant-size token alphabet and a constant number of learnable parameters, but a non-constant numerical precision.
The only step using non-constant numerical precision is our implementation of the parity primitive. In this step, attention is used to aggregate a fixed bit-coordinate of the tokens in the current context window and to determine its parity; our realization of this operation uses logarithmic-precision arithmetic.

This precision requirement is specific to this particular implementation of the parity primitive, not to the parity-HIST reduction itself. The reduction only requires access to a primitive that returns the parity of the relevant histogram coordinate. 
If this primitive is implemented directly with bounded-precision operations, for example, by allowing arithmetic over $\mathbb{F}_2$ or wrap-around arithmetic, then the rest of the simulation remains bounded-precision as well. Recent work, e.g., \citep{kozachinskiy2026parity} has studied parity computation by transformers under related but different architectural assumptions.  We therefore separate this issue from our main reduction: the potential bounded-precision implementation of the histogram-parity primitive in the sliding-window setting without using PE is left outside the scope of this paper.
\end{remark}


\section{Related Work}



\paragraph{Positional information in transformers.}
The original Transformer architecture uses positional encodings to inject token-order information into an otherwise permutation-equivariant attention stack \citep{vaswani2017attention}.  
Common positional mechanisms for transformers include absolute positional encodings, such as sinusoidal positional encodings and learned positional embeddings \citep{vaswani2017attention}, as well as relative positional mechanisms, such as relative positional representations and T5-style relative position biases \citep{shaw2018self,raffel2020exploring}; rotary position embeddings \citep{su2024roformer}; and attention-bias methods such as ALiBi \citep{press2022train}.
PE affect both empirical behavior and formal expressiveness.  For example, \citet{salzer2025counting} studied the counting ability of transformers without neither positional information nor sliding window mechanism. PE is also closely related to length generalization: extrapolating beyond the training length can change model behavior in ways that depend on the positional-encoding scheme \citep{kazemnejad2023impact}.
Our work is complementary to this line of research.  Instead of designing a new positional mechanism, we show that PEs are not necessary for universality when the model uses a sliding context window.  The key point is that, although the computation inside each window can be permutation-invariant, the FIFO-style motion of the window over time provides a sequential signal.  

\paragraph{Transformer universality and chain-of-thought.}
Attention-based architectures are known to be Turing complete under some architectural assumptions \citep{perez}.  Recent work has refined this picture by studying how chain-of-thought, precision, depth, and window size affect the computational power of transformers \citep{merrill2023expressive,prompt,depth,pencil}.  Closest in spirit are universality constructions for transformers with constant precision and short window by simulating Post machines \citep{constant,tm2}.  All these constructions use PE as an indispensable resource, 
and in some cases, PE also act as an additional source of computational power, making it less clear whether the source of universal computation lies in the transformers content-based computation or in the PE attached to it.
For instance, 
\cite{tengyu} uses PE to store the structure of the circuit to be computed, and \citep{constant,tm2} employ PE, depending on the simulated space bound, to single out the unique token that should be attended to.
Our contribution is to remove explicit PE from the universality proof by routing the simulation through a permutation-invariant window-summary model.

\section{Preliminaries}
For a finite alphabet \(\Sigma\) and a sequence \(w=(w_1,\ldots,w_m)\in\Sigma^*\), let
\[
    |w|=m,
    \qquad
    \hist(w)=\bigl(\#\{j\in[m]:w_j=\sigma\}\bigr)_{\sigma\in\Sigma}\in\N^{\Sigma}.
\]
We write \(u\concat v\) for concatenation and \(\tail(w)\) for the word obtained by deleting the first symbol of a nonempty word \(w\).  Let \(\mathbf e_\sigma\) denote the standard basis vector indexed by \(\sigma\).  Define
\[
    \Pi(w) := \bigoplus_{j=1}^m \mathbf e_{w_j}\in\zo^{\Sigma},
\]
where \(\oplus\) denotes component-wise addition modulo two.  Thus \(\Pi(w)=\hist(w)\bmod 2\).

\subsection{Post machine}
A Post machine (PM) \citep{post1} is defined as a tuple $\langle\Sigma,Q,\delta\rangle$ where 
\begin{itemize}
\item $\Sigma$ is the tape alphabet, including a designated ``blank'' symbol $\perp$, a designated ``start'' symbol $\triangleright$, and numbers 0 and 1.
\item $Q$ is a finite set consisting of possible states. We assume $Q$ contains a designated start state $q_{start}$ and a designated halting state $q_{halt}$.
\item $\delta: Q \times \Sigma \to Q \times (\Sigma\cup\Sigma^2)$ is the transition function, specifying how the machine updates the tape and state.
\end{itemize}
The machine has access to an unbounded queue whose entries are symbols from $\Sigma$. The queue is initialized to contain the start symbol $\triangleright$ and a finite input string $x\in\{0,1\}^*$. The machine starts in state $q_{start}$. The computation repeatedly performs the following step until the machine enters $q_{halt}$. Suppose the current state is $q\in Q$, and the symbol at the front of the queue is $\sigma\in\Sigma$. If
\[
\delta(q,\sigma)=(q',\sigma')
\]
for some $q'\in Q$ and $\sigma'\in\Sigma\cup\Sigma^2$, then the machine updates its state from $q$ to $q'$, removes the front symbol $\sigma$ from the queue, and appends the string $\sigma'$ to the rear of the queue.

\paragraph{Turing completeness of PM} 
The model of PM is Turing complete, i.e., equivalent to TMs in terms of computational power. Specifically, 
\begin{theorem}\label{thm:pmtm}
Given any single-tape TM running in $t$ time and $s$ space, there is an equivalent PM running in $O(t\times s)$ time and using a $s$-size queue.
\end{theorem}
This theorem has been implicitly proved in several places, see, e.g., Appendix A in \citep{constant}.

\subsection{Transformers without PE}
In the following definition of transformers without PE, for notational simplicity, we only consider the single-head, single-layer transformer, which suffices for our construction.

Let $\mathcal{V}$ be a finite vocabulary. A decoder-only transformer is an autoregressive sequence model mapping $\mathcal{V}^* \to \mathcal{V}$. Given a sequence $(v_1, \dots, v_i) \in \mathcal{V}^i$, it computes the next token $v_{i+1}$ via the following layers:
\begin{enumerate}
    \item \textbf{Token Embedding:} Maps each $v_j \in \mathcal{V}$ to a vector $\mathrm{emb}(v_j) \in \mathbb{R}^D$, where $D$ is the embedding dimension.
    \item \textbf{Positional Encoding:} There is no PE, and thus the initial input representation is exactly $h_j^0 := \mathrm{emb}(v_j)$.
    \item \textbf{Self-Attention Layer:} 
    Using parameter matrices $Q, K, V \in \mathbb{R}^{D \times D}$, the unmasked attention logits are computed as $\mathsf{logit}_{i, j} = \langle Q h_i^0, K h_j^0 \rangle$. The attention weights $\alpha_{i,j}$ are obtained via the standard softmax function \textit{over the context window} $C = (v_{i - |C| + 1}, \dots, v_{i})$; namely, 
    $$
        \alpha_{i,j} := \frac{\exp(\mathsf{logit}_{i, j} )}{\sum_{j = i - |C| + 1}^i \exp(\mathsf{logit}_{i, j})} \ \text{ for $j = i - |C| + 1, \dots, i$.} 
    $$
    The output is $a_i = \sum_{j = i - |C| + 1}^i \alpha_{i,j} V h_j^0$, and the residual stream becomes $h_i^{0.5} = a_i + h_i^0$.
    \item \textbf{Feed-Forward Network (MLP):} A multi-layer ReLU network $f_{\mathrm{mlp}}$ is applied, yielding $h_i^1 = f_{\mathrm{mlp}}(h_i^{0.5}) + h_i^{0.5}$.
    \item \textbf{Output Layer:} The next token is generated deterministically by 
    $$
        v_{i+1} = \operatorname*{argmax}_{v \in \mathcal{V}} (\mathrm{emb}(v)^\top h_i^1).
    $$
\end{enumerate}
\paragraph{Adaptive context window and CoT generation.} We consider a model where the transformer uses an adaptive sliding-window update during generation.  
In the beginning, on input sequence $v = (v_1, \dots, v_n)$, the context window is set to be the entire prompt sequence. 
During generation, every vocabulary token \(w\) carries an update bit \(\gamma(w)\in\zo\).  
After generating \(w\), the context window is updated by
\[
    C \gets 
    \begin{cases}
        C\concat w, & \gamma(w)=1,\\
        \tail(C)\concat w, & \gamma(w)=0.
    \end{cases}
\]
Thus \(\gamma=0\) is the usual sliding update, while \(\gamma=1\) appends a token without evicting an old token and hence increases the window length $|C|$ by one.

\paragraph{Positional encoding.}
In an \emph{absolute} PE scheme, each absolute position \(j\) is assigned a vector \(p_j\), and the embedding layer is replaced by \(h_j^0=\emb(v_j)+p_j\); standard examples include sinusoidal absolute encodings and learned absolute position embeddings \citep{vaswani2017attention,radford2019language}.  In a \emph{relative} PE scheme, the attention computation is modified by a quantity depending on the offset \(i-j\), for example, by adding a relative bias or a relative key/value vector to the score or value computation \citep{shaw2018self}.  Rotary position embeddings implement a related relative mechanism by applying position-dependent rotations to queries and keys, so that their inner products encode relative offsets \citep{su2024roformer}.  

\section{The HIST Model and Its Turing Completeness}\label{sec:hist}
A \emph{HIST machine} is defined as a tuple $(\Sigma, Q, \delta, q_{start}, q_{halt})$ where 
\begin{itemize}
\item $\Sigma$ is the tape alphabet, including a designated ``blank'' symbol $\perp$, a designated ``start'' symbol $\triangleright$, and numbers 0 and 1.
   \item $Q$ is a finite set consisting of possible states. We assume $Q$ contains a designated start state $q_{start}$ and a designated halting state $q_{halt}$.
   \item $\delta: Q\times \N^{\Sigma}\rightarrow Q\times \left(\Sigma\cup \Sigma^2\right)$ is a transition function describing the rules used in performing each step of the HIST machine.
\end{itemize}

The machine has a single tape that is infinite to the right and bounded on the left. The tape has two heads, called the front head and the rear head. Initially, from the leftmost cell rightward, the tape contains the start symbol $\triangleright$, followed by a finite input string $x\in\{0,1\}^*$, and then blank symbols $\perp$ on all remaining cells. 
The front head starts at cell right next to $\triangleright$ (i.e., pointing to the first position of $x$), the rear head starts at the right end of the input, and the machine starts in state \(q_{\mathrm{start}}\). We call the interval of tape cells between the front and rear heads, including both endpoints, the window. 

The computation repeatedly performs the following step until the machine enters $q_{\mathrm{halt}}$. Suppose the current state is $q\in Q$, and the histogram of the symbols in the current window is $\vec h\in\mathbb{N}^{\Sigma}$. If
\[
\delta(q,\vec h)=(q',\sigma')
\]
for some $q'\in Q$ and $\sigma'\in \Sigma\cup\Sigma^2$, then the machine updates its state from $q$ to $q'$, moves the front head one cell to the right, moves the rear head $|\sigma'|$ cells to the right, and writes the string $\sigma'$ in the newly exposed cell(s) at the right end of the tape.

A HIST machine is a \emph{parity-HIST machine} if its transition depends on the histogram only through its parity vector $(\bm{h} \bmod 2) \in  \zo^\Sigma$. Equivalently, the transition can be viewed as a map
$
  \delta:  Q\times\zo^\Sigma \to Q\times(\Sigma\cup\Sigma^2).
$

\begin{theorem}[Parity-HIST simulates Post machines]\label{thm:pm-to-hist}
For every Post machine \(\PM\), there exists a parity-HIST machine \(\HIST\) that simulates \(\PM\) step by step.  More precisely, if \(C_t\) denotes the queue of \(\PM\) after \(t\) Post-machine steps, then immediately before simulating the \((t+1)\)-st Post-machine step, the HIST window is \(\tail(C_t)\). 
\end{theorem}

\begin{proof}
Let
$
    \PM=(\Sigma,Q_{\PM},\delta_{\PM},q_{\mathrm{start}},q_{\mathrm{halt}})
$
be a Post machine.  Recall that the Post-machine queue is initialized as \(\triangleright x\), while the HIST window is initialized as \(x=\tail(\triangleright x)\).  The missing front symbol \(\triangleright\) is handled by a one-bit flag in the finite control.

We construct a parity-HIST machine over the same alphabet \(\Sigma\).  Its finite control stores three pieces of information: the current Post-machine state, the parity of the current Post-machine queue, and a one-bit flag indicating whether the first transition is being simulated.  Formally, define
\[
    Q_{\HIST}
    =
    \bigl(Q_{\PM}\times\zo^\Sigma\times\zo\bigr)
    \cup \{q_{\mathrm{halt},\HIST}\}.
\]
The initial HIST state is
$
    (q_{\mathrm{start}},\mathbf 0,1),
$
where the last bit records that the first deleted Post-machine symbol is known to be \(\triangleright\).  
The halting state is \(q_{\mathrm{halt},\HIST}\).

We now define the parity-HIST transition.  Suppose the current HIST state is \((q,r,b)\), where \(q\in Q_{\PM}\), \(r\in\zo^\Sigma\), and \(b\in\zo\).  Let
$
    p=\hist(W)\bmod 2
$
be the parity vector of the current HIST window \(W\).
If \(b=1\), then this is the first simulated Post-machine step, and the deleted queue symbol is
$
    \triangleright.
$
If \(b=0\), then \(r\) is the parity of the current Post-machine queue \(C\), while \(p\) is the parity of \(\tail(C)\).  Therefore,
$
    r\oplus p=\mathbf e_\sigma,
$
where \(\sigma\) is the front symbol of \(C\).  On configurations that do not satisfy this invariant, define the transition arbitrarily.

Having recovered \(\sigma\), the HIST machine applies the Post-machine transition
$
    \delta_{\PM}(q,\sigma)=(q',a).
$
The HIST machine appends exactly the same sequence \(a\).  The parity of the next Post-machine queue is
$
    p\oplus \Pi(a),
$
because the next queue is \(W\concat a\).  Thus, if \(q'=q_{\mathrm{halt}}\), the next HIST state is \(q_{\mathrm{halt},\HIST}\); otherwise, the next HIST state is
$
    \bigl(q',\,p\oplus\Pi(a),\,0\bigr).
$
This transition depends on the current window \(W\) only through \(p=\hist(W)\bmod 2\), so the constructed machine is indeed a parity-HIST machine.

We prove correctness by induction on the number of simulated Post-machine steps.  Let \(C_t\) be the queue of $\msf{PM}$ after \(t\) steps and let $W_t$ denote the window of $\msf{Hist}$ after $t$ steps. Initially,
$
    C_0=\triangleright x,
    W_0=x=\tail(C_0).
$
The first HIST transition uses the flag \(b=1\) to recover the deleted symbol \(\triangleright\), applies the first transition of \(\PM\), and appends the same word \(a_0\).  Hence
$
    W_0\concat a_0
    =
    x\concat a_0
    =
    C_1.
$
After the HIST window update, the new window is \(\tail(C_1)\), as required.

For the induction step, assume that before simulating the \((t+1)\)-st Post-machine step, the HIST window is
$
    W_t=\tail(C_t),
$
the simulated Post-machine state is correct, and the stored parity is \(r=\Pi(C_t)\).  Write
$
    C_t=\sigma\concat W_t,
$
where \(\sigma\) is the front symbol of the Post-machine queue.  Since the HIST machine observes
$
    p=\Pi(W_t),
$
it recovers
$
    \mathbf e_\sigma
    =
    r\oplus p.
$
Therefore, it applies the same Post-machine transition
$
    \delta_{\PM}(q,\sigma)=(q',a_t)
$
and appends the same word \(a_t\).  The Post-machine queue after this step is
$
    C_{t+1}=W_t\concat a_t,
$
which is exactly the word obtained by appending \(a_t\) to the current HIST window.  After the automatic HIST window update, the next HIST window is
$
    \tail(W_t\concat a_t)=\tail(C_{t+1}).
$
The stored parity is updated to
$
    p\oplus\Pi(a_t)=\Pi(C_{t+1}),
$
so the induction invariant is preserved.  This completes the proof.
\end{proof}

\begin{corollary}\label{cor:hist-tc}
Parity-HIST machines are Turing complete.
\end{corollary}

\begin{proof}
Combine \cref{thm:pm-to-hist} with \cref{thm:pmtm}.
\end{proof}

\section{Simulating the HIST Model with a Sliding-Window Transformer}\label{sec:tf-sim}
We now show that every Parity-HIST machine can be implemented by a sliding-window transformer without positional encodings, in the sense of the transformer definition in Section~2.  The construction uses the adaptive update bit from the transformer model: if a HIST transition appends one symbol, the transformer emits one token with update bit \(0\); if a HIST transition appends two symbols, the transformer emits the first token with update bit \(1\) and the second token with update bit \(0\).  Thus a two-symbol append is implemented as one append-without-deletion update followed by one ordinary sliding update.

The only numerical issue is parity extraction.  Uniform attention gives normalized histogram coordinates, not raw counts.  We therefore introduce a special marker symbol \(\#\).  The simulation maintains the invariant that the context contains exactly one marker, except immediately after the marker has slid out and before it is refreshed.  When exactly one marker is present, its normalized histogram mass is \(1/|C|\), so the model can infer \(|C|\) and convert normalized symbol masses into raw counts.  With \(O(\log S)\)-bit inference precision for a run whose HIST window length is at most \(S\), the parity of each count can be recovered.

\begin{theorem}[Simulating Parity-HIST by sliding-window transformers]\label{thm:transformer-simulates-hist}
Every Parity-HIST machine can be faithfully simulated by a sliding-window transformer without positional encodings.  More precisely, for every execution in which the HIST window length is at most \(S\), the construction uses context length at most \(S+1\) and \(O(\log S)\)-bit inference precision.
\end{theorem}

\begin{proof}[Construction]
Let
$
    \HIST=(\Sigma,Q_{\HIST},\delta,q_{\mathrm{start}},q_{\mathrm{halt}})
$
be a Parity-HIST machine.  Let
$
    \Sigma_{\#}:=\Sigma\cup\{\#\},
$
where \(\#\notin\Sigma\) is a marker symbol used only by the transformer simulation.

\paragraph{Vocabulary.}
The transformer vocabulary is
$
    \mathcal V=\Sigma_{\#}\times R\times\zo.
$
A token \(v=(\tau,r,\gamma)\in\mathcal V\) consists of a displayed symbol \(\tau\in\Sigma_{\#}\), a control label \(r\in R\), and an update bit \(\gamma\in\zo\).  The control-label set is
\[
    R=
    \{\msf{N}(q):q\in Q_{\HIST}\}
    \cup
    \{\msf{P}(q,\sigma):q\in Q_{\HIST},\ \sigma\in\Sigma\}.
\]
The normal label \(\msf{N}(q)\) means that the next transformer generation step should simulate a HIST transition from state \(q\).  The pending label \(\msf{P}(q,\sigma)\) means that the first symbol of a two-symbol HIST append has already been emitted, the simulated HIST state has become \(q\), and the next transformer token must emit the pending second symbol \(\sigma\).

For a context
\[
    C=(v_{i-m+1},\ldots,v_i),
    \qquad
    v_j=(\tau_j,r_j,\gamma_j),
\]
define \(\msf{sem}(C)\in\Sigma^*\) to be the word obtained by deleting every marker symbol \(\#\) from \((\tau_{i-m+1},\ldots,\tau_i)\), while preserving the order of all symbols in \(\Sigma\).  The marker is not part of the simulated HIST window.

If the initial HIST window is \(W_0=w_1\cdots w_m\), then the transformer prompt is
\[
    (w_1,\msf{N}(q_{\mathrm{start}}),0),\ldots,
    (w_m,\msf{N}(q_{\mathrm{start}}),0),
    (\#,\msf{N}(q_{\mathrm{start}}),0).
\]
Here \(W_0\) denotes the HIST window defined in \cref{sec:hist}, namely the tape segment excluding the front-head cell; in particular, on raw input \(x\), the initial HIST window is \(W_0=x\).  Thus the initial semantic projection is \(W_0\), and the newest token carries the initial control label.

\paragraph{Embeddings and attention matrices.}
Let \(N=|\Sigma_{\#}|\) and 
write
$
    D=2N+ |R| + 2.
$
The embedding has four blocks: a displayed-symbol block, a histogram block, a control-label block, and an update-bit block.  For \(v=(\tau,r,\gamma)\), define
\[
    \emb(v)= ( \mathbf e_{\tau}^{N},  \mathbf 0^{N}, \mathbf{e}^{|R|}_r, \mathbf{e}_{\gamma + 1}^2)^\top 
    \in\R^D.
\]
There are no positional encodings, so at generation step \(i\), for every token \(v_j\) in the current context, we have 
$
  h_j^0=\emb(v_j).
$

We choose the self-attention matrices as follows.  The query and key matrices are all-zero matrices:
$
    \mathbf Q_{\mathrm{att}}= 
    \mathbf K_{\mathrm{att}}=\mathbf 0^{D\times D}.
$
The value matrix copies the displayed-symbol block into the histogram block and zeros all other blocks:
\[
    \mathbf V_{\mathrm{att}} \cdot \emb(\tau,r,\gamma)=
  ( \mathbf 0^N,  \mathbf{e}_{\tau}^{N}, \mathbf{0}^{|R|}, \mathbf{0}^2)^\top.
\]

Suppose that the current context at generation step \(i\) is
\[
    C_i=(v_{i-m+1},\ldots,v_i),
    \qquad
    v_j=(\tau_j,r_j,\gamma_j), 
\]
By the construction of the embedding layer, 
$
    h_j^0
    = (\mathbf{e}_{\tau_j}^N, \mathbf{0}^N, r_j,  \mathbf{e}_{\gamma_j + 1}^2)^\top
$
for $j = i - m + 1, \dots, i$.
Since \(\mathbf Q_{\mathrm{att}}=\mathbf K_{\mathrm{att}}=0\), every attention logit in the current context is 0 and hence the softmax weights are uniform:
$
    \alpha_{i,j}=\frac1m,
$ for all $ j=i-m+1,\ldots,i$.
The attention output at the newest position is
\[
\begin{aligned}
    a_i
    =\sum_{j=i-m+1}^{i}\alpha_{i,j}\mathbf V_{\mathrm{att}}h_j^0  
    = (\mathbf{0}^N, A_i, \mathbf{0}^{|R|}, \mathbf{0}^2)^\top.
\end{aligned}
\]
where
$
    A_i(\rho)=\frac{\#\{j\in\{i-m+1,\ldots,i\}:\tau_j=\rho\}}{m}, \ \rho\in\Sigma_{\#}.
$
Using the residual connection in the self-attention layer, the representation after attention is
\[
    h_i^{0.5}=a_i+h_i^0
    =  ( \mathbf e_{\tau}^{N}, A_i,  \mathbf{e}_{r_i}^{|R|},  \mathbf{e}_{\gamma_i + 1}^2)^\top.
\]
Thus, the second block of \(h_i^{0.5}\) is the normalized displayed-symbol histogram of the current context, while the third block retains the control label of the newest token.

\paragraph{The MLP layer.}
If the context contains exactly one marker, then \(A_i(\#)=1/m\).  Hence, for every \(\sigma\in\Sigma\),
\[
    \operatorname{cnt}_i(\sigma)
    :=\frac{A_i(\sigma)}{A_i(\#)}
    =\#\{j\in\{i-m+1,\ldots,i\}:\tau_j=\sigma\}.
\]
With \(O(\log S)\)-bit precision for counts up to \(S+1\), the MLP recovers
\[
    p_i=\bigl(\operatorname{cnt}_i(\sigma)\bmod 2\bigr)_{\sigma\in\Sigma}
    =\hist(\msf{sem}(C_i))\bmod 2.
\]
If \(A_i(\#)=0\), the marker is absent, and the MLP uses the marker-refresh rule below instead of evaluating a HIST transition.

Since normalize histogram $A_i$ and the window length $m$ can be read off from $h_{i}^{0.5}$, the MLP implements a deterministic map \(G\) that readily simulates the transition of $\msf{Hist}$
and maintain that there is always exactly one $\#$ symbol in the context window.
The specific definition of $G$ and the proof that the generated sequence simulates the HIST execution are deferred to \cref{app:tf-proof}.
\end{proof}

\begin{corollary}\label{cor:no-pe-tc}
Sliding-window transformers without positional encodings are Turing complete.  More precisely, if a single-tape Turing machine runs in time \(t\) and space \(s\), then the transformer simulation uses \(O(ts)\) generated tokens and context length \(O(s)\).
\end{corollary}

\begin{proof}
Combine \cref{thm:pmtm,thm:pm-to-hist,thm:transformer-simulates-hist}.
\end{proof}


\section{Conclusions and Limitations}\label{sec:conclusion}
We have shown that positional encoding is not necessary for the computational universality of Transformers in the sliding-window autoregressive regime. Although a Transformer without PE is permutation-invariant with respect to the tokens inside a single window, the temporal motion of the window provides an additional source of sequential structure. This FIFO-like structure is sufficient to simulate the dequeue operation of a Post machine, and hence to achieve Turing completeness.

Our result should not be interpreted as showing that Transformers without PE can recover exact token positions inside the context window. Our construction deliberately avoids positional addressing by introducing the parity-HIST model, where the transition depends only on a finite control state and the parity of the current window histogram. The result, therefore, shows that exact positional access is stronger than what is needed for universal computation: a permutation-invariant window summary, together with finite control and sliding-window dynamics, already suffices.

An interesting future direction is to investigate whether the $s$-factor slowdown in our construction is inherent. Our result shows that positional encoding is not necessary for universal computation, but it remains unclear whether removing PE necessarily incurs an efficiency cost. Establishing either a faster simulation or a matching lower bound would further clarify the computational role of positional information in transformers.



{\small
\bibliographystyle{alpha}
\bibliography{reference,reference2}
}

\appendix

\section{Correctness of the Transformer Simulation}\label{app:tf-proof}

\paragraph{The MLP layer.}
The MLP layer implements a map $G$ defined as follows.
For valid representations \(h_i^{0.5}\), \(G\) is defined as follows.

First, if \(A_i(\#)=0\), the marker has just slid out.  If the newest control label is \(r_i\), set
\[
    G(h_i^{0.5})=(\#,r_i,0).
\]
Second, if \(A_i(\#)>0\) and \(r_i=\msf{P}(q,\sigma)\), set
\[
    G(h_i^{0.5})=(\sigma,\msf{N}(q),0).
\]
Third, if \(A_i(\#)>0\) and \(r_i=\msf{N}(q)\), then the MLP evaluates the parity-HIST transition
$
    \delta(q,p_i)=(q',a),
$
where $a \in \Sigma \cup \Sigma^2$.
If \(a=\sigma_1\) has length one, set
$
    G(h_i^{0.5})=(\sigma_1,\msf{N}(q'),0).
$
If \(a=\sigma_1\sigma_2\) has length two, set
$
    G(h_i^{0.5})=(\sigma_1,\msf{P}(q',\sigma_2),1).
$
The case \(q=q_{\mathrm{halt}}\) can be assigned any fixed self-loop behavior, since the simulation is read when the halting control label is reached.

We choose the MLP output so that
\[
    f_{\mathrm{mlp}}(h_i^{0.5})=\emb(G(h_i^{0.5}))-h_i^{0.5}.
\]
Consequently the residual connection in the MLP layer gives
\[
    h_i^1=f_{\mathrm{mlp}}(h_i^{0.5})+h_i^{0.5}
    =\emb(G(h_i^{0.5})).
\]
Finally, the next token is
\[
    v_{i+1}
    =\operatorname*{argmax}_{v\in\mathcal V}\emb(v)^\top h_i^1
    =G(h_i^{0.5}).
\]
The last equality is unique: \(\emb(G(h_i^{0.5}))\) has one nonzero coordinate in each of the displayed-symbol, control-label, and update-bit blocks, and no vocabulary token other than \(G(h_i^{0.5})\) matches all three blocks.

\paragraph{Correctness of Simulation.}
Let
\[
    W_0,W_1,W_2,\ldots
\]
be the sequence of HIST windows, and let \(q_t\) be the HIST state before the transition from \(W_t\).  For every non-halting step, write
\[
    \delta(q_t,\hist(W_t)\bmod 2)=(q_{t+1},a_t),
    \qquad a_t\in\Sigma\cup\Sigma^2.
\]
By the definition of the HIST update, the next HIST window is
\[
    W_{t+1}=\tail(W_t\concat a_t).
\]

We prove that the transformer maintains the following invariant.  At the beginning of the simulation of HIST step \(t\), the context is in one of two forms:
\begin{enumerate}
    \item \emph{Ready form:} the context contains exactly one marker, \(\msf{sem}(C)=W_t\), and the newest token carries the normal label \(\msf{N}(q_t)\);
    \item \emph{Refresh form:} the context contains no marker, the newest token carries the normal label \(\msf{N}(q_t)\), and \(\tail(\msf{sem}(C))=W_t\).  The next transformer generation step is a marker-refresh step.  This form can occur only immediately after the previous simulated HIST update removed the marker.
\end{enumerate}
The prompt is in ready form for \(t=0\), because it encodes the initial HIST window and appends one marker token carrying \(\msf{N}(q_{\mathrm{start}})\).

Suppose first that the context is in refresh form.  By construction, \(A_i(\#)=0\), so the MLP outputs \((\#,\msf{N}(q_t),0)\).  The update bit \(0\) deletes the oldest semantic token and appends a new marker.  Hence the new semantic projection is \(\tail(\msf{sem}(C))=W_t\).  The resulting context contains exactly one marker and has newest label \(\msf{N}(q_t)\).  Thus the context is restored to ready form.

Now assume the context is in ready form for step \(t\).  The attention computation in the body gives
\[
    h_i^{0.5}
    =
    \begin{pmatrix}
        \mathbf e_{\tau_i}^{N}\\
        A_i\\
        \mathbf{e}_{\msf{N}(q_t)}\\
        \mathbf{e}_{\gamma_{i} + 1}
    \end{pmatrix},
\]
where \(A_i\) is the normalized displayed-symbol histogram of the current context.  Since the context contains exactly one marker, the MLP recovers
\[
    \hist(\msf{sem}(C))\bmod 2
    =\hist(W_t)\bmod 2.
\]
It therefore evaluates the same transition as the parity-HIST machine:
\[
    \delta(q_t,\hist(W_t)\bmod 2)=(q_{t+1},a_t).
\]

If \(a_t=\sigma_1\) has length one, the transformer emits
$
    (\sigma_1,\msf{N}(q_{t+1}),0).
$
After the adaptive update, there are two possibilities.  If the removed token is a semantic token, then the marker remains present and the new semantic projection is
\[
    \tail(W_t\concat\sigma_1)=W_{t+1}.
\]
The context is ready for step \(t+1\).  If the removed token is the marker, then the context has no marker and semantic projection \(W_t\concat\sigma_1\).  The next transformer step is the refresh step described above; after it, the semantic projection becomes
\[
    \tail(W_t\concat\sigma_1)=W_{t+1},
\]
and the context is ready for step \(t+1\).

If \(a_t=\sigma_1\sigma_2\) has length two, the transformer first emits
\[
    (\sigma_1,\msf{P}(q_{t+1},\sigma_2),1).
\]
Because the update bit is \(1\), no token is removed.  The marker remains present, the semantic projection becomes \(W_t\concat\sigma_1\), and the newest token carries the pending label \(\msf{P}(q_{t+1},\sigma_2)\).  On the next generation step, the pending rule emits
\[
    (\sigma_2,\msf{N}(q_{t+1}),0),
\]
without recomputing the HIST transition.  After this sliding update, if a semantic token is removed, then the marker remains present and the semantic projection is
\[
    \tail(W_t\concat\sigma_1\sigma_2)=W_{t+1}.
\]
If instead the marker is removed, then the context has no marker and semantic projection \(W_t\concat\sigma_1\sigma_2\); the subsequent refresh step appends a new marker and deletes the oldest semantic token, producing
\[
    \tail(W_t\concat\sigma_1\sigma_2)=W_{t+1}.
\]
In both cases, after at most one refresh step, the context is ready for step \(t+1\) and the newest label is \(\msf{N}(q_{t+1})\).

By induction, each HIST transition is simulated exactly.  A one-symbol append uses one transformer token plus possibly one refresh token, and a two-symbol append uses two transformer tokens plus possibly one refresh token.  Hence the number of generated tokens is linear in the number of HIST steps.  If the HIST window length is at most \(S\), then the transformer context length is at most \(S+1\): it consists of the semantic window plus one marker in ready form, and the intermediate append-without-deletion case for a two-symbol append is also bounded by \(S+1\) because the resulting HIST window length is at most \(S\).  This proves the claimed faithful simulation.
\end{document}